\documentclass[12pt,oneside,reqno]{amsart}

\usepackage{graphicx}
\usepackage{pict2e}
\usepackage{amssymb}
\usepackage{amsthm}                
\usepackage[margin=1in]{geometry}  
\usepackage{graphics}
\usepackage{epstopdf}
\usepackage{amsmath}
\usepackage{graphicx}
\usepackage{subfigure}
\usepackage{latexsym}
\usepackage{amsmath}
\usepackage{amsfonts}
\usepackage{amssymb}
\usepackage{tikz}
\usepackage{mathdots}
\usepackage{comment}
\usepackage{float}
\usepackage[mathscr]{euscript}
\usepackage{enumerate}
\usepackage{epsfig}
\usepackage { graphicx }
\usepackage { hyperref }
\usepackage { graphics }
\usepackage { graphicx }
\usepackage{MnSymbol} 
\usepackage{booktabs}
\usepackage{wrapfig,lipsum,booktabs}

\usepackage{caption}

\usepackage{algorithm}

\usepackage{etoolbox}

\usepackage {eufrak}
\usepackage[utf8]{inputenc}
\usepackage{longtable}
\usepackage[lite]{amsrefs}

\renewcommand{\PrintDOI}[1]{\href{http://dx.doi.org/\detokenize{#1}}{doi: \detokenize{#1}}%
	\IfEmptyBibField{pages}{, (to appear in print)}{}}

\theoremstyle{definition}
\newtheorem{theorem}{Theorem}[section]
\newtheorem{lemma}[theorem]{Lemma}

\newtheorem{proposition}[theorem]{Proposition}

\theoremstyle{definition}
\newtheorem{definition}[theorem]{Definition}
\newtheorem{example}[theorem]{Example}
\theoremstyle{remark}
\newtheorem{remark}[theorem]{Remark}
\numberwithin{equation}{section}

\numberwithin{equation}{section}

\title{A TOPOLOGICAL FRAMEWORK FOR DEEP LEARNING}
\author{Mustafa Hajij}
\address{Santa Clara University}
\email{mhajij@scu.edu}
\author{Kyle Istvan}
\email{KyleIstvan@gmail.com}

\date{}

\dedicatory{}

\begin{document}

	\maketitle 
	
\begin{abstract}
We utilize classical facts from topology to show that the classification problem in machine learning is always solvable under very mild conditions. Furthermore, we show that a softmax classification network acts on an input topological space by a finite sequence of topological moves to achieve the classification task. Moreover, given a training dataset, we show how topological formalism can be used to suggest the appropriate architectural choices for neural networks designed to be trained as classifiers on the data. Finally, we show how the architecture of a neural network cannot be chosen independently from the shape of the underlying data. To demonstrate these results, we provide example datasets and show how they are acted upon by neural nets from this topological perspective.


\end{abstract}


\section{Introduction}

The purpose of this article is to give a high level description of the role topology plays when considering the action of a classification neural network on the domains of its target data's components. This work is driven by the observation that a neural network is essentially a composition of continuous functions. Given that topology provides a rigorous study of continuous functions and the spaces upon which they operate, it is natural to look at neural networks from this perspective.

Denote by $M^n$ to a manifold of dimension $n$ and let $X = M_1^{i_1} \cupdot M_2^{i_2} \cdots  \cupdot M_k^{i_k}$ be a disjoint union of $k$ compact manifolds immersed in some Euclidean space $\mathbb{R}^n$. Let $S$ be a dataset sampled from $X$. We are interested in the following two questions: 

\begin{enumerate}
    \item Suppose that the data set $S$ is labeled. What are the topological constraints the manifolds $M_{k}^{i_k}$ place on the architecture of a neural network defined on $X$ and trained using $S$?
    \item How do the topology and geometry of the activation sets of a classification neural network change as we pass from one layer in the network to the next?
\end{enumerate}

We begin by providing a few definitions that describe the learning problem as a topological one. Once this framework is set, the results follow readily. More specifically, we utilize classical tools from topology to show that the classification problem is always solvable under very mild topological conditions. We show that a softmax classification network acts on an input topological space by a finite sequence of topological moves to realize the classification task. Furthermore, we show how the architecture of a neural network cannot be chosen independently without careful consideration of the shape of the input data. Finally, we demonstrate these results, we provide example datasets and show how they are acted upon by neural nets, from this topological perspective.

\section{Previous Work}
 Over the last decade multiple connections have been made between topology and machine learning.  Perhaps most notable among these is the field of Topological Data Analysis (TDA) \cite{edelsbrunner2010computational,carlsson2009topology}. This includes mixing topological signatures with deep neural networks \cite{hofer2017deep,bruel2019topology,wangtopogan}.
 
 One of the common themes in TDA is that data itself plays a central role in the learning task, and attempts to uncover methods by which underlying structural characteristics of the data itself might be discovered.  A strong motivation for finding such methods is the principle that any learning model we create should not be designed independently of the underlying data. Our goal here is related to TDA from this perspective. 
 
 Whereas TDA is concerned with understanding the structure of data itself, our work aims at determining which architectural choices are appropriate when designing a neural network, subject to certain assumptions on the data.  One might think of it as existing downstream of TDA in the production pipeline; once certain characteristics of a dataset have been uncovered, we use that understanding to inform the design of potential learning models using that dataset.
 
Alternatively, our work here can be regarded as part of the effort in the literature regarding the explainablity of deep learning \cite{hagras2018toward,selvaraju2017grad}. The authors Zeiler et. al. in \cite{zeiler2014visualizing} introduced a visualization technique that gives insight into the intermediate layers of convolutional neural networks. In \cite{yosinski2015understanding} also gives a way to visualize and interpret the a given convolutional network by looking at the activations. Furthermore, Li et. al. \cite{lei2020geometric} demonstrated that natural high dimensional data concentrates close to a low-dimensional manifold and provided experimental evidences showing that the success of deep learning is probably due to the manifold structure in real data. TDA was also used for in \cite{gabrielsson2019exposition} to aid in the interpretation of a deep neural network.

 The earliest hints, that we know of, related to our work appears in a blog by C. Olah \cite{olah2014neural}. Olah performed a number of topological experiments illustrating the importance of considering the topology of the underlying data when making a neural network. In \cite{naitzat2020topology} the activations of a binary classification neural network were considered as point clouds that the layer functions of the network are acting on. The topologies of these activations are then studied using homological tools such as persistent homology \cite{edelsbrunner2010computational}. The problem when studying the activation space as a point cloud is that it is difficult to draw precise conclusions about the behaviour of a neural network in general. On the other hand, our work takes a more rigorous stance and we directly study the topological spaces that the neural networks act on and from which the activations point cloud are sampled. With our topological formalism, for instance, we can rigorously prove statements about the classification problem in general and in particular within the context of neural networks. In \cite{carlsson2020topological} the construction of Mapper \cite{rathore2019topoact} was utilized on the weights of convolutional deep neural networks to develop an understanding of the computations that they perform. In \cite{rathore2019topoact} Rathore et. el. also utilized the Mapper construction to study the shape of the the activations of a neural network.  However, Mapper parameters are in general difficult to tune, and in the latter case, they must also be tuned independently for each layer activation. Moreover, it is hard to interpret the meaning of multiple Mapper graphs on consecutive activations of a network because each Mapper of each activation is studied independently from the other Mapper graphs and there is no clear correspondence between multiple graphs. Finally, its not clear what insights one may generalize to generic neural networks from these visualizations.

Overall, quantitative analysis seems to provide the dominant toolset currently used to interpret any given neural network and gain confidence in its performance.

In the present article we clearly distinguished between data and the functions that operate on it. We believe that this distinction is important because data as a separate mathematical object have complex properties that intertwine non-trivially with the functions, that also have unique properties, that operate on the data. Our work here is to study this intertwining nature between data and the functions that operate on it. We hope to supplement this by providing a rigorous topological framework in which neural networks may be studied.

\section{Supervised Machine Learning From A Topological Perspective}
\label{first section}
The purpose of the next two sections is to study the classification problem in a topological setting. Supervised learning problems are typically presented in a statistical context.  We start by outlining the problem in a precise statistical setting, in which the topological nature of both the dataset and the learning models is not traditionally present.  We then recast the classification problem in a topological environment in order to demonstrate the importance of these topological considerations.   

\subsection{Supervised Machine Learning: Statistical Setting}
\label{stat}
An \textit{instance} $x$ is a vector of a Euclidean space $\mathbb{R}^n$. A \textit{label} is a point from an arbitrary set $\mathcal{Y}$. In the classification setting the set $\mathcal{Y}$ is finite, whereas in the regression setting $\mathcal{Y}$ is infinite, typically some Euclidean space.  Denote by $\mathcal{X}\subset \mathbb{R}^n$ to the space of all instances. Let $P(x,y)$ be an unknown joint probability distribution on $\mathcal{X}\times \mathcal{Y}$. Let $S = \{(x_i,y_i)\}_{i=1}^n$ be a training dataset sampled from the probability distribution $P(x,y)$. In supervised classification we are seeking a function $f:\mathcal{X}\longrightarrow \mathcal{Y}$ with the goal is that $f(x)$ predicts the true label $y$ on a future $x$ where $(x,y) \sim P(x,y) $. 



Typically, finding such a function $f$ is done by defining a \textit{cost function} $c$ that penalizes the deviation of predicted labels ${f(x_i)}$ from the true labels ${y_i}$. With this setting a best function $f^{*}$ is one that minimizes the expected value of this cost function; that is, 

\begin{equation}
    f^*= argmin_{f\in \mathcal{F}}E_{(x,y)\sim P } c(x,y,f(x) )
\end{equation}
 where $\mathcal{F}$ is \textit{hypothesis space}, the space of all possible functions $f:\mathcal{X} \longrightarrow \mathcal{Y}$ that we are willing (or able) to consider in a given learning problem. Typically the space $\mathcal{F}$ is a strictly smaller space than  $C(\mathcal{X},\mathcal{Y})$ the space of all functions from $\mathcal{X}$ to $\mathcal{Y}$.


\subsection{Supervised Machine Learning: Topologically Setting}
We now present the topological definition corresponding to the statistical setting presented in Section \ref{stat}. We restrict ourselves to the classification setting where the set $\mathcal{Y}$ of labels is finite.



Let $X = M_1^{i_1} \cupdot M_2^{i_2} \cdots  \cupdot M_k^{i_k}$ be a disjoint union of $k$ compact manifolds. Let $h: X \longrightarrow \mathbb{R}^n$ be a continuous function on $X$. Denote by $\mathcal{X}$ to the image set $h(X)$ of the manifold $X$ inside $\mathbb{R}^n$.  We refer to the pair $(X,h)$ as \textit{topological data}. A \textit{topological labeling} on $\mathcal{X}$ is subset $\mathcal{X}_L \subseteq \mathcal{X} $ that can be written as a finite collection of closed disjoint sets $\mathcal{X}_1,\cdots \mathcal{X}_d$ in $\mathcal{X}$. Written differently, a topological labeling on $\mathcal{X}$ is a closed subset $\mathcal{X}_L \subseteq \mathcal{X} $ along with a surjective continuous function $g : \mathcal{X}_L \to \mathcal{Y}$ where $\mathcal{Y} = \{l_1, \cdots , l_d\}$ is a finite set given the discrete topology. The triplet $(X,h,g)$ will be called \textit{topologically labeled data}. 

A few remarks here must be made about the above definition. The set $X$ represents an abstract topological manifestation of the data. The space $\mathbb{R}^n$ corresponds in the statistical setting to the ambient space of a probability distribution $\mu$ from which we sample the data. The support of $\mu$ is the set $ h(X):=\mathcal{X}$. 

The assumption that the data lives on a manifold-like structure is known in the literature  \cite{fefferman2016testing,lei2020geometric,korman2018autoencoding}. While we make this assumption here, it not strictly necessary anywhere in our proofs. More importantly, in our definition above we clearly distinguish between the set $X$ and its image $\mathcal{X}$ in some Euclidean space given by the function $h: X \longrightarrow \mathbb{R}^n$ \footnote{Doing computations on manifolds typically require having a parameterization of that manifold. Note that in our setting $h$ is not required to be a parameterization, only continuity is assumed.}. We make this distinction for many reasons. First, a topological data is in essence just merely the space $X$. However, in practice to do computations on $X$ one must have a \textit{manifestation} of this data inside some Euclidean space. The continuity of the map $h$ guarantees that some aspects of $X$ are preserved \footnote{This is in contrast with a  parameterization map, or a chart, of a manifold where all local aspects of the manifold $X$ are preserved via this map.}, but it is really the proprieties of $X$ itself that are of interest. Second, as we will demonstrate later, the dimension of the Euclidean space $\mathbb{R}^n$ will put constraints on the model required to describe certain aspects of $X$.

With the above setting we now demonstrate how to realize the classification problem as a topological problem. In what follows we set $\mathcal{X}_j$ to denote $g^{-1}(l_j)$ for $l_j \in \mathcal{Y}$.

\begin{definition}
\label{def1}
Let $(X,h,g)$ be topologically labeled data with, $h: X \longrightarrow \mathcal{X} \subset \mathbb{R}^{d_{in}}$ and $g:  \mathcal{X}_L \subseteq \mathcal{X}\longrightarrow \mathcal{Y} $ where $|\mathcal{Y}|=d_{out}$. Let $f:\mathbb{R}^{d_{in}}\longrightarrow \mathbb{R}^k $ be a continuous function. We say that $f$ \textit{separates} the topologically labeled data $(X,h,g)$ if we can find $d_{out}$ disjoint embedded $k$-dimensional discs $D_1, \cdots , D_{d_{out}}$ in $\mathbb{R}^k$ such that $f(\mathcal{X}_j)\subset D_j$ for $1 \leq j\leq d_{out}$.
\end{definition}

The preceding description is an abstract rewording of the classification problem given in Section \ref{stat}. A successful classifier tries to \textit{separate} the labeled data by mapping the raw input data into another space where this data can be separated easily according to the given class. This mapping is represented by the function $f$ in Definition \ref{def1} where this function maps the space $\mathcal{X}$ from its ambient space $\mathbb{R}^{d_{in}}$ to another ambient space $\mathbb{R}^k$ such that $f(\mathcal{X}_L)$ can be separated by $k$-dimensional embedded disks and each disk contains a subset of $f(\mathcal{X}_L)$ that corresponds to a single label from $\mathcal{Y}$.

The function $f$ represents the learning function that we usually like to compute in practice.  Note that within this setting, the classification problem presented in Section \ref{stat} is now purely topological. It follows that we expect the tools that we will utilize to study the classification problem are also topological in nature. The first question, which we will address in the following section, is one of existence: given topologically labeled data $(X,h,g)$ when can we find a function $f$ that separates this data?  

\begin{example}

In general, a topologically labeled data can be knotted, linked and entangled together in a non-trivial manner by the embedding $h$, and the existence of a function $f$ that separates this data is not immediate. See Figure \ref{unlinking}. 

\begin{figure}[h]
  \centering
   {\includegraphics[width=0.4\textwidth]{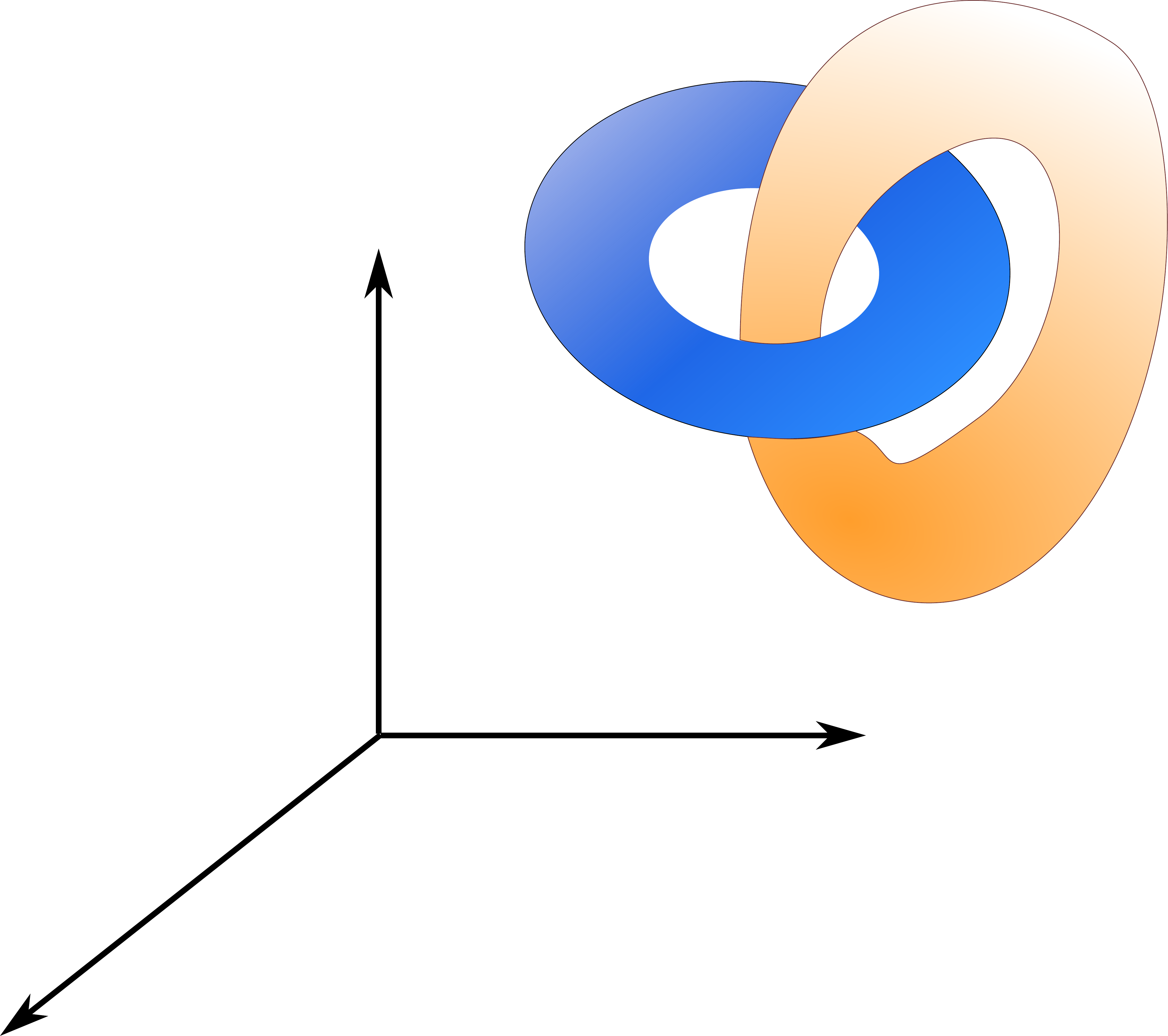}
    \caption{An example of linked topologically labeled data. }
  \label{unlinking}}
\end{figure}

\end{example}


\begin{remark}
Note that the topological setting presented above is consistent with the statistical setting given in \ref{stat}. Indeed, a \textit{statistical data} is defined to be a function $\hat{g}: S \to \mathcal{Y} $ where $S \subset \mathcal{X}$ is finite set. However, the questions that we address here are independent from the restriction $\hat{g}$. In other words, regardless of statistical data that we might be given, what effects do the original space $X$ and the choice of the embedding $h$ have on the final classification problem?
\end{remark}

\section{Urysohn's Lemma  and Supervised Classification Problems}
This topological setting yields immediate insights into the classification problem.  Towards this goal, recall the following theorem from basic topology, which for the purposes of clarity we do not present in its fullest generality.

\begin{lemma}
(Urysohn’s Lemma) Let $A,B$ be disjoint closed subsets of a normal topological space $X$. There exists a continuous function $f: X \longrightarrow [0,1] $ such that $f(A)=0$ and $f(B)=1$.
\end{lemma}
We start with the binary classification problem, namely when $|\mathcal{Y}|=2$. We have the following proposition:

\begin{proposition}
\label{one}
Let $(X,h,g)$ by a topologically labeled data with $h: X \to \mathcal{X} \subset  R^{d_{in}}$ and $g: \mathcal{X}_L \subset \mathcal{X} \to \{l_1, l_2\}$. Then there exists a continuous map $f: \mathbb{R}^{d_{in}} \to \mathbb{R}$ that separates $(X,h,g)$.
\end{proposition}
\begin{proof}
The label function  $g: \mathcal{X}_L \longrightarrow \{l_1, l_2\}$ induces a partition on $\mathcal{X}_L$ into two disjoint sets $\mathcal{X}_1$ and $\mathcal{X}_2$ simply by setting $\mathcal{X}_i=g^{-1}(l_i)$, for $i=1,2$.   By definition the sets $\mathcal{X}_1$ and $\mathcal{X}_2$ are closed in $\mathcal{X}$ and disjoint. By Urysohn's lemma there exists a function $f^*:\mathcal{X} \longrightarrow [0,1]$ such that $f^*(\mathcal{X}_1)=0$ and $f^*(\mathcal{X}_2)=1$. Since $\mathcal{X}$ is closed in $\mathbb{R}^{d_{in}}$ then by Tietze extension theorem there exists an extension of $f^*$ to a continuous function $f : \mathbb{R}^{d_{in}} \to \mathbb{R} $ such that $f^*(\mathcal{X})=f(\mathcal{X})$.  In particular, $f(\mathcal{X}_1)=0$ and $f(\mathcal{X}_2)=1$. Hence the function $f$ separates $(X,h,g)$.
\end{proof}

Proposition \ref{one} can be easily generalized to obtain functions that separate $(X,h,g)$ in any Euclidean space $\mathbb{R}^k$. Namely, for any $k\geq 1$ there exists a continuous map $F: \mathbb{R}^{d_{in}} \longrightarrow \mathbb{R}^{k}$ that separates $(X,h,g)$. This can be done by defining $F = (f_1 , f_2)$ where $f_1:\mathbb{R}^{d_{in}}\longrightarrow [0,1]$ is the continuous function guaranteed by Proposition \ref{one} and $f_2:\mathbb{R}^{d_{in}}\longrightarrow \mathbb{R}^{k-1}$ is an arbitrary continuous function. This function $F$ clearly separates  $(X,h,g)$. We record this fact in the following proposition.
\begin{proposition}
\label{second}
Let $(X,h,g)$ by a topologically labeled data with  $h: X \to \mathcal{X} \subset R^{d_{in}}$ and $g: \mathcal{X}_L \subset \mathcal{X} \to \{l_1, l_2\}$. Then for any $k\geq 1$ there exists a continuous map $f: \mathbb{R}^{d_{in}} \to \mathbb{R}^{k}$ that separates $(X,h,g)$.
\end{proposition}

Proposition \ref{second} can be generalized to the case when the set $\mathcal{Y}$ has an arbitrary finite size. We first need the following generalization of Urysohn’s Lemma on $n$ sets.

\begin{lemma}
\label{general}
Let $A_1,A_2, \cdots  A_n$ be closed and mutually disjoint sets in a normal space $X$.
Then there exists a continuous map $f: X \longrightarrow \mathbb{R}$ such that $f(A_i) = i$ for $0 \leq i \leq n$.
\end{lemma}
\begin{proof}
 Consider $f_i$ be continuous functions to $[0,1]$ with $f_i( A_1 \cup \cdots \cup A_{i-1} )=0$ and $f_i(A_{i}  \cup \cdots \cup A_{n})=1$, guaranteed by Urysohn's Lemma. Then $f=\sum_{i=1}^n f_i $ a continuous function with the desired properties.
\end{proof}
Combining Lemma \ref{general} and Proposition \ref{second}, we make the following theorem asserting the existence of a function $f$ that separates any given topologically labeled data.

\begin{theorem}
\label{generalization TLD}
Let $(X,h,g)$ be topologically labeled data with  $h: X \to \mathcal{X}\subset R^{d_{in}}$ and $g: \mathcal{X}_L \subset \mathcal{X} \to \mathcal{Y}$.  Then there exists a continuous map $f: \mathbb{R}^{d_{in}} \to \mathbb{R}^k$ that separates $(X,h,g)$ for any integer $ k \geq 1$.
\end{theorem}

Uryson's lemma holds also in the smooth and PL-categories. In other words, if the space $X$ is a smooth or a PL-manifold then we can also find a smooth or a PL- function $f$ that satisfies the above condition. This makes the above results applicable with working with these categories as well. 


\subsection{Topologically Separable Spaces and Universality of Neural Networks}
So far our discussion has been purely topological and we have not discussed the relationship between the topological framework and neural networks. We give a few definitions and show that this setup yields immediate consequences when considering the capabilities of a neural network.

A \textit{neural network}, or simply a \textit{network}, is a function $Net: \mathbb{R}^ {d_{in}} \longrightarrow \mathbb{R}^{d_{out}}$ defined by a composition of the form: 
\begin{equation}
\label{Net}
    Net:=f_{L} \circ \cdots \circ f_{1}
\end{equation}
where the functions $f_{i}$, $1 \leq i \leq L $ called the \textit{layer functions}. A layer function $f_i:\mathbb{ R }^{n_i} \longrightarrow \mathbb{ R }^{m_i} $ is typically a continuous, a piece-wise smooth or a smooth function of the following form: $f_i(x)=\sigma (W_i(x)+b_i)$ where $W_i$ is  an $m_i\times n_i$ matrix, $b_i$ is a vector in $\mathbb{R}^{m_i} $, and $\sigma :\mathbb{R}\longrightarrow \mathbb{R} $ is an appropriately chosen nonlinear function that is applied coordinate-wise on an input vector $(z_1,\cdots,z_{m_i} ) $ to get a vector $(\sigma( z_1),\cdots,\sigma(z_{m_i}))$. In the previous section we showed that for any topologically labeled data we can find a continuous function $f$ that separates this data. A natural question to ask given a topologically labeled data $(X,g,h)$ can we always find a network \ref{Net} that separates this data? A network of the form \ref{Net} is clearly a continuous function since it's a composition of such functions but it is not immediately clear the function $f$ can be chosen to be of the form \ref{Net}.


While the space of all neural networks is relatively small in comparison to the space of all continuous functions, it is dense inside the space of all continuous functions with respect to an appropriately chosen functional norm, thanks to the so called universality of neural networks \cite{cybenko1989approximations,hanin2017approximating,lu2017expressive} \footnote{The universal approximation theorem is available in many flavors : one may fix the depth of the network and vary the width or the other way around.}. The universality of neural networks essentially states that for any continuous function $f$ we can find a network that approximates it to an arbitrary precision\footnote{The closeness between functions is with respect to an appropriate functional norm. See \cite{cybenko1989approximations,lu2017expressive} for more details. }. Hence we conclude that any topologically labeled data can effectively be separated by a neural network. 



\subsection{Shape Of Data and Neural Networks}
We discuss now briefly how the input shape of data is essential when deciding on the architecture of the neural network. Consider for instance the data depicted in Figure \ref{unlinking}. This data consists of two links in $\mathbb{R}^3$. The reader might convince herself, at least intuitively, that one needs at least $4$ dimension in order to unlink this space. We talk more about this point in Section \ref{conclusions}.

Theorem \ref{2222} shows if we are not careful about the choice of the first layer function of a network then we can always find topologically labeled data that cannot be separated by this network.

\begin{theorem}
\label{2222}
Let $Net$ be neural  network  of  the  form :
$Net=Net_1 \circ f_1$
with $f_1:\mathbb{R}^n\longrightarrow \mathbb{R}^k$ such that $f_1(x)= \sigma( W(x)+b)$ and $k<n$ and $Net_1 : \mathbb{R}^k\longrightarrow \mathbb{R}^d $ is an arbitrary net. Then there exists a topologically labeled data $(X,h,g)$ with $h:X \to \mathcal{X} \subset \mathbb{R}^n$  and $g: \mathcal{X}_L  \subset \mathcal{X}\longrightarrow \mathbb{R}^d $ that is not separable by $Net$.
\end{theorem}
\begin{proof}

Let $X = \mathcal{X}= \{x\in\mathbb{R}^n, ||x||\leq 2 \}$.  Let  $\mathcal{X}_L =\mathcal{X}_1 \cupdot \mathcal{X}_2 $ where $\mathcal{X}_1=\{x\in\mathbb{R}^n, ||x||\leq 0.9 \}$ and $\mathcal{X}_2=\{x\in\mathbb{R}^n, 1 \leq ||x||\leq 2 \}$.  Choose $g:\mathcal{X}_L \longrightarrow \{l_1,l_2\}$ such that $g(\mathcal{X}_1)=l_1$ and $g(\mathcal{X}_2)=l_2$. Let $f_1$ be a function as defined in the Theorem. The matrix $W : \mathbb{R}^n \longrightarrow \mathbb{R}^k $ where $k < n$ has a nontrivial kernel. Hence, there is a non-trivial vector $v \in \mathbb{R}^n$ such that $W(v)=0$. Choose a point $p_1 \in M_1 $ and a point $p_2 \in M_2$ on the line that passes through the origin and has the direction of $v$. We obtain $W(p_1)=W(p_2)=0$. In other words, $f_1(p_1)=f_1(p_2)$. Hence $Net(p_1)=Net(p_2)$ and hence $Net(M_1) \cap Net(M_2) \neq \emptyset $ and so we cannot find two embedded disks that separate the sets $Net(M_1)$, $Net(M_2)$. 

\end{proof}

\begin{remark}
Note that in Theorem \ref{2222} the statement is independent of the depth of the neural  network. This is also related to the work \cite{johnson2018deep} which shows that skinny neural networks are not universal approximators. This is also related to the work in \cite{nguyen2018neural} where is was shown that a network has to be wide enough in order to successfully classify the input data. 
\end{remark}

\section{Action of a classification Neural Network on A Topological Space}
\label{action}
Given a network of the form \ref{Net}, we would like to consider how this network acts on the input topological space and deforms it as we pass from one layer to the next. For that we need first to establish a few definitions and notations.

 Consider a network $Net$ defined by the composition $Net=f_L \circ \cdots \circ f_1 $. For $0 \leq i \leq L $, we define the \textit{ $i^{th}$ head} of $Net$ to be 
\begin{equation}
    Net^{[i]}(x)=f_i \circ \cdots \circ f_2 \circ f_1 (x).
\end{equation}
We set $Net^{[0]}(x)=x$.  A layer function $f_{i}$ operates on its input $x^{[i]}$ and produces an output $x^{[i+1]}=f_{i}(x^{[i]})$. We will denote the initial input to the network by $x^{[1]}$. Note that by our convention we have $Net^{[i]}(x)=x^{[i+1]}$. Moreover, $Net^{[L]}(x)=Net(x)$.

By a slight abuse of notation, we will use $\mathcal{X}$ to denote the domain $\mathbb{R}^{n_0}$ of $Net$. 
 We want to study the function $Net$ by understanding how the topology and geometry of the elements of the sequence $\{\mathcal{X}^{[i]}:=Net^{[i]}(\mathcal{X}) | 0 \leq i \leq L \}$ change as we move through the network (right to left in the function composition).

We approach this by considering the following two goals:

\begin{enumerate}
    \item Understanding how each individual block acts, as each layer function act on its input domain. We address this point in \ref{firstpoint}.
    \item Understanding what the neural network as whole is trying to accomplish as a continuous function. This point is addressed this in \ref{secondpoint}.
\end{enumerate}
 To this end, we need to specify the type of layer functions we work with. One of the most popular layer functions is the Relu layer function. This takes the form $f_i(x)=Relu(W_i(x)+b_i)$ where $W_i$ is an $m\times n$ matrix, $b_i$ is a vector in $\mathbb{R}^m $, and $Relu:\mathbb{R}^m\longrightarrow \mathbb{R}^m $ is the Rectified Linear Unit activation function \cite{krizhevsky2012imagenet} defined by $Relu(x_1,\cdots,x_m)=(max(x_1,0),\cdots,max(x_m,0))$. Note that we distinguish between the Relu layer function, which is the function $f_i$, and the Relu activation function which is the function $Relu$. A network that consists of Relu layer functions, with the possible exception of the final layer, will be called a \textit{Relu network}.

\subsection{Relu Layer Functions as Topological Operations}
\label{firstpoint}
We now consider the action of the Relu activation function on topological spaces in $\mathbb{R}^n$.

Set $A \subset \mathbb{R}^n$. The Relu activation function $Relu : \mathbb{R}^n \longrightarrow \mathbb{R}^n$ acts on the
set $A$ in one or more of the following three ways:  

\begin{enumerate}
\item Quotienting  : The image $Relu(A)$ is obtained from $A$ by identifying certain parts of $A$, and is thus a quotient of $A$. In other words, $Relu(A)=A/\sim$ where $x\sim y$ if and only $Relu(x)=Relu(y)$.
\item Bending : In this case $Relu(A)$ is obtained from $A$ by bending the set $A$ at some locations. In this case $A$ and $Relu(A)$ are homeomorphic.
\item The identity action : In this case the $Relu$ function acts trivially on the set $A$ and $Relu(A)=A$. This occur when the coordinates of $x \in A$ are all positive.   
\end{enumerate}
Figure \ref{Relu} demonstrates three examples of the action of Relu on various sets in $\mathbb{R}^2$.

\begin{figure}[h]
  \centering
   {\includegraphics[width=\textwidth]{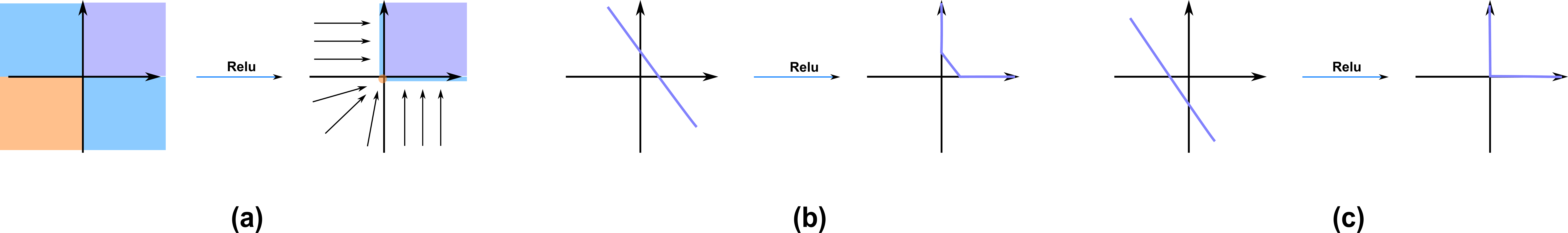}
    \caption{Example of the action of Relu function on sets in $\mathbb{R}^2$. (a) The Relu function acts on the entire plane $\mathbb{R}^2$ as follows : for the first quadrant this function acts as the identity, for the second this function projects all points to the $y$ axis, for the third quadrant everything gets mapped to the origin and finally for the fourth quadrant all points are projected to the $x$-axis. In this case, the Relu function does all the operations listed above : quotienting, bending and the identity action.  (b) This is an example of a bending. (c) This is an example of both quotienting and bending. }
  \label{Relu}}
\end{figure}

On the other hand a matrix $W$ acts linearly on a subset of $\mathbb{R}^n$ in a straightforward manner. If $A$ has full rank then $W(A)$ is homemorphic to $A$, and $W$ acts on the set $A$ via rotation, reflection and scaling. These transformations can be directly interpreted from the singular value decomposition of $W$.

If $W$ does not have full rank, then it also acts as a quotient map. More specifically, the set $W(A)$ is a quotient space of $A$, where pairs of points whose difference lies in the kernel (nullspace) of $A$ are identified. 

Finally, adding the vector $b$ simply constitutes a translation, leaving the underlying topological structure unchanged.

Putting all the above observations together and thinking of each type of the above continuous operations as a \textit{topological move} \footnote{The notion of topological moves is very common topology. It is usually utilized to describe a set of continuous operations, called moves, that one can utilize to move  a topological object from one state to another. One famous example is Reidemeister moves which are a set of topological moves that act on knots to transfer them from one state to another. }, one may think about the action of a neural network as a a sequence of finite moves : scaling, translation, rotation, reflection, bending and quotienting.

To address the second point mentioned in Section \ref{action} we need to specify the type of neural networks that we wish to consider, which in our case will be classification neural networks. Classification neural networks typically have a special layer function at the end where one uses the \textit{softmax activation function} instead of the Relu function (there are other types of classification neural networks but this is beyond the scope of our discussion here). Denote by $\Delta_n$ the $n^{th}$ simplex defined as : $\Delta_n =\{ (x_0,\cdots,x_{n+1} ) \in \mathbb{R}^{n+1}| \sum_{i=0}^{n+1} x_i=1, x_i \geq 0  \}$. Note that $\Delta_n$ is the convex hull of the points $\{v_0,\cdots ,v_{n} \}$ where $v_i=(0,...,1,...,0)\in\mathbb{R}^{n+1}$ with the lone $1$ in the $(i+1)^{th}$ coordinate. The points $v_i$, for $0\leq i \leq n$ are usually called the vertices of the simplex $\Delta_n$.

  The softmax function $softmax:\mathbb{R}^{n} \longrightarrow Int (\Delta_{n-1}) \subset \mathbb{R}^{n}$, 
defined by the composition $D \circ Exp  $ where $Exp : \mathbb{R}^n\to (\mathbb{R}^+)^n  $ is defined by : $ Exp ( x_1,\cdots,x_n  ) = ( \exp( x_1), \cdots,\exp( x_n)  ) $, and $D :\mathbb{R}^n \to \Delta_{n-1} $ is defined by :

$$ D ( x_1,\cdots,x_n  ) = ( x_1/\sum_{i=1}^n x_i, \cdots, x_n /\sum_{i=1}^n x_i ). $$ maps the entire Euclidean space $\mathbb{R}^n$ to the $(n-1)^{th}$ simplex $\Delta_{n-1}$. Usually $n$ is the number of labels in the classification problem. Each vertex $v_i$ in $\Delta_{n-1}$ corresponds to precisely one label $l_{i+1} \in \mathcal{Y} $ for $0 \leq i \leq n-1 $. We will call a Relu network that has a softmax layer at the end a \textit{softmax classification neural network.}

 We are now ready to consider a simple example.

\begin{example}
\label{example123}

Consider the dataset sampled from a torus embedded in $\mathbb{R}^3$ , labelled by the three colors yellow, purple, and green as indicated in  Figure \ref{second_example}. To make the problem more interesting we separate some of the yellow labeled point with the purple ones such that a purple ring is sandwiched between two yellow rings.
We train this network on the above dataset and trace the activations as demonstrated in Figure \ref{second_example}

Suppose that we want to use a softmax classification neural network to classify this data. Furthermore we wish to trace how the individual continuous layer functions that form the neural network act on the input dataset. To this end let $Net$ be a softmax classification neural network with 6 layers in which all layers map $\mathbb{R}^3$ to itself.

\begin{figure}[h]
  \centering
   {\includegraphics[width=0.99\textwidth]{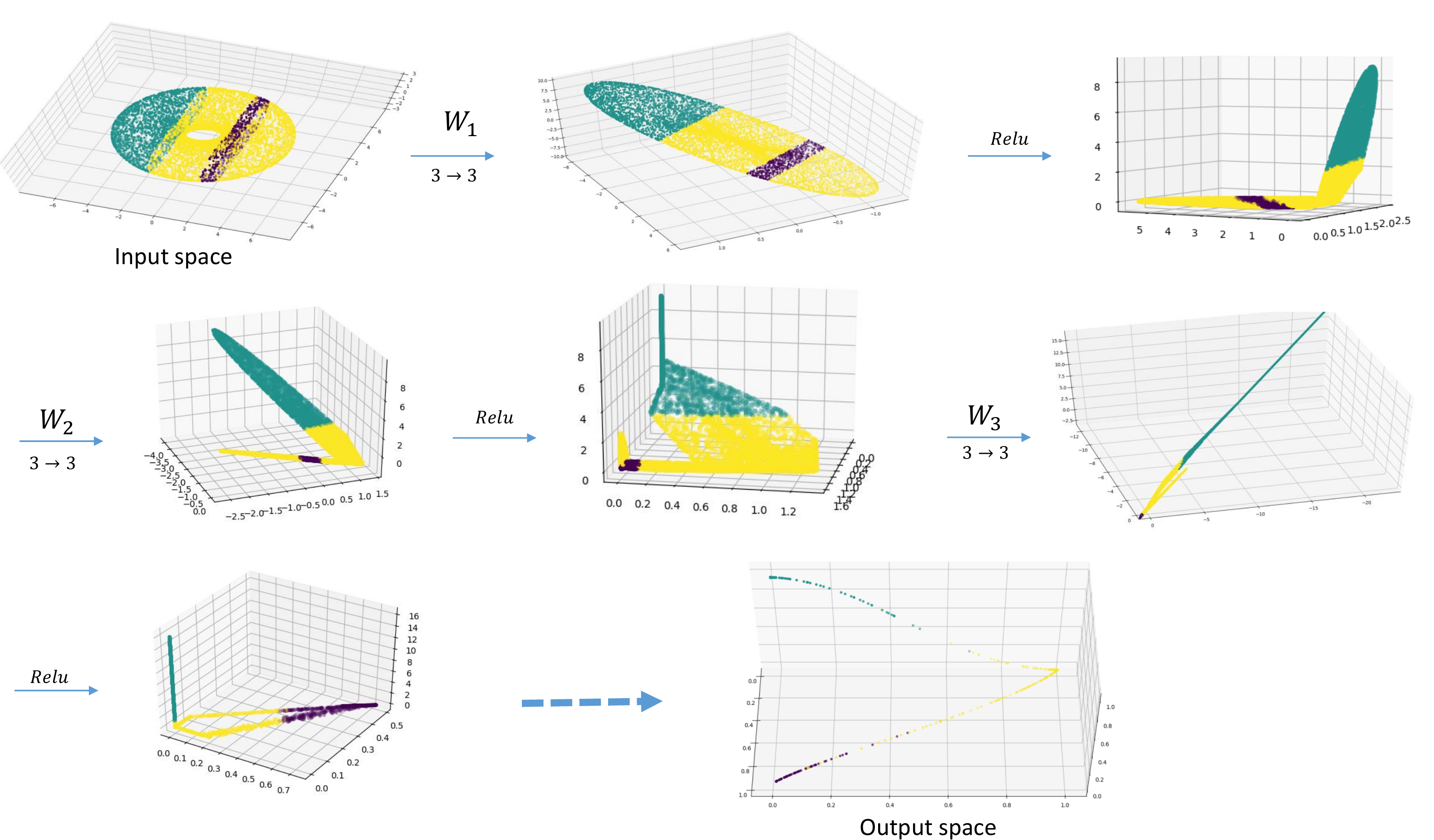}
    \caption{A neural network acting on torus with three labels.}
  \label{second_example}}
\end{figure}

Denote by $\mathcal{X}$ to the torus dataset. Figure \ref{second_example} shows the action of the function of the first few layers on the input dataset. We make the following observation :


\end{example}

\begin{enumerate}
    \item The matrix $W_1$ projects the torus to a disk. This is a quotienting operation, creating a two dimensional topological space.
    \item The Relu function bends the disk.
    \item The matrix $W_2$ reorients the space, and the following Relu function bends it in such a way that two yellow subsets can subsequently be identified via projection.
    \item The matrix $W_3$  quotients the space into a roughly  one-dimensional topological space.
    \item The Relu function that comes after $W_3$ bends the space and make the yellow points close to each other.
    \item The final few layers, not depicted, deform the space to fit inside $\Delta_2$. Note that the (previously disconnected) yellow components are now identified together in one single region.
\end{enumerate}

\begin{example}
\label{examples1}


Suppose that we want to use a classification neural network to classify the data given on the top left part in Figure \ref{annulus}. To this end let $Net$ be a neural network given by the composition $Net=f_6\circ f_5\circ f_4\circ f_3\circ f_2\circ f_1$. For $1\leq i \leq 5$, the maps are given by $f_i := Relu(W_i(x)+b_i)$ such that $W_1:\mathbb{R}^2\longrightarrow \mathbb{R}^5 $, $W_2:\mathbb{R}^5\longrightarrow \mathbb{R}^5 $, $W_3:\mathbb{R}^5\longrightarrow \mathbb{R}^2 $
and $W_j : \mathbb{R}^2\longrightarrow \mathbb{R}^2 $ for $ 4 \leq j \leq 5 $. Finally, the function, $f_6 = softmax(W_6(x)+b_6)$  where  $W_6:\mathbb{R}^2\longrightarrow \mathbb{R}^2 $. 

We train this network on the above dataset and trace the activations as demonstrated in Figure \ref{annulus}. In the Figure we visualize the activations in higher dimension by projecting them using Isomap \cite{tenenbaum2000global} to $\mathbb{R}^3$. Our choice of this algorithm as a dimensionality reduction algorithm is driven by the fact that the dataset we work with here is essentially a manifold; as such, projecting the space to a lower dimension with the Isomap algorithm should preserve most of the topological and geometric structure of the this space.  

\begin{figure}[h]
  \centering
   {\includegraphics[width=0.99\textwidth]{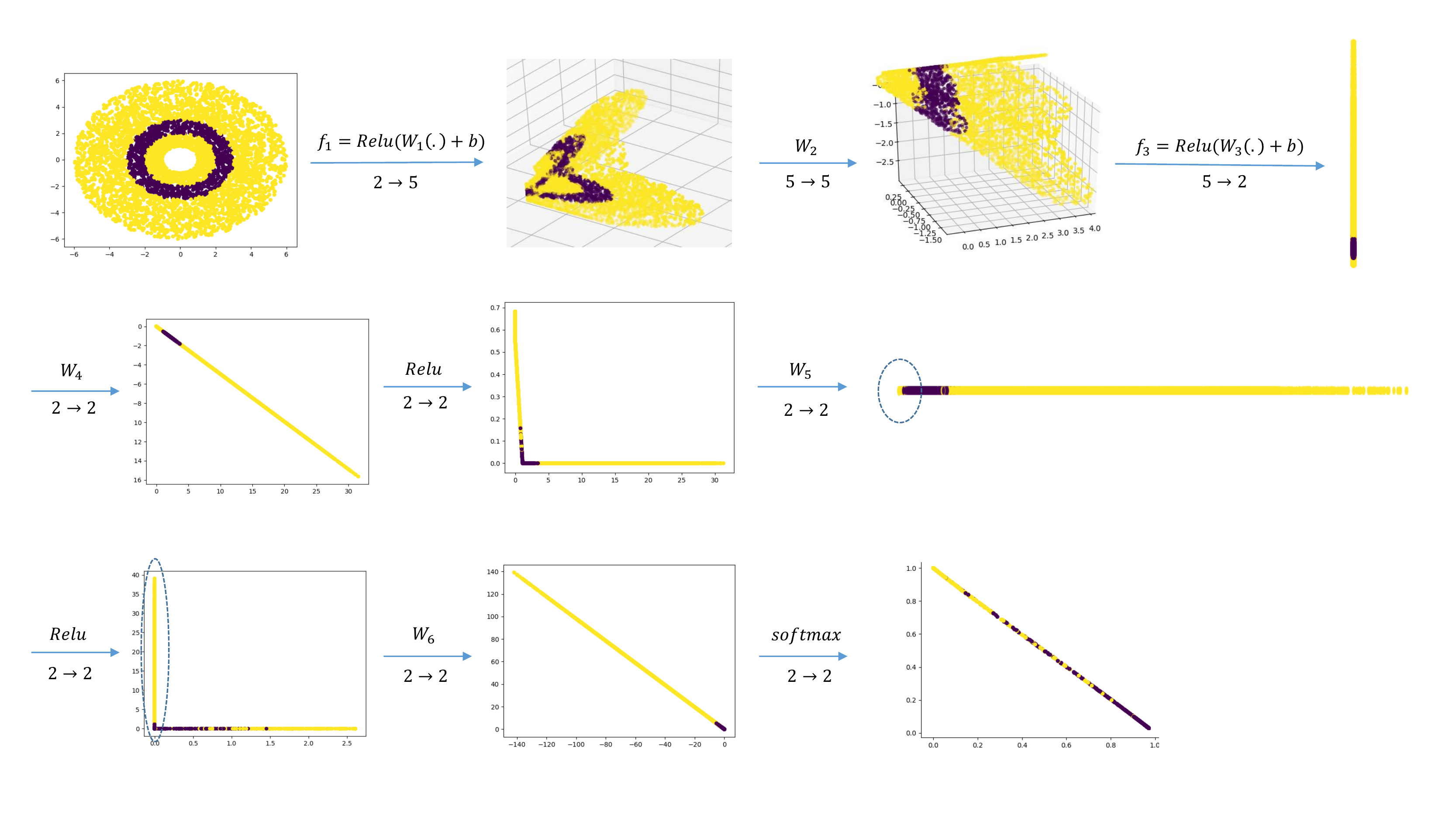}
    \caption{The topological operations performed by a network on data sampled from the annulus and colored by two lables. }
  \label{annulus}}
\end{figure}
Inspecting the activations in Figure \ref{annulus} we make the following observation: 

\begin{enumerate}
    \item A neural network can quotient the topological space either by using nonlinearity or by utilizing the linear part of a given layer function. This is the case with the map $f_3: \mathbb{R}^5\longrightarrow \mathbb{R}^2$. While the linear component is a projection onto $\mathbb{R}^2$, the network "chose" to project the space into $1-$ manifold since the second dimension is not needed for the final classification.
    \item  Note that the yellow components are separated by the purple one, and in order to map both of these parts to the same part of the space, the net has to glue these two parts together. Indeed, the neural network quotients parts of the space as it sees it necessary. This is visible in $W_5$, which acts as a projection, and again $W_6$. 
\end{enumerate}

\end{example}

\subsection{Collapsing the input topological space to the softmax codomain space}
\label{secondpoint}
As seen in Figure \ref{second_example}, we can think of the layer functions in a given neural network as a finite set of topological operations that act in sequence to deform the input space  $\mathcal{X}^{[0]}$ into the final space $\mathcal{X}^{[L]}$.

In general we do not have control over the shape of the input space $ \mathcal{X}^{[0]}$. On the other hand, the final shape $\mathcal{X}^{[L]}$ is determined by the final activation layer which in our case is the $n-1$ simplex $\Delta_{n-1}$, where $n$ is the number of labels in the classification problem.  Figure \ref{softmax} illustrates  topologically labeled data, which is essentially an annulus, with three labels. The corresponding codomain of the neural network that we may build to classify data sampled from this space will be $\Delta_2$ which is a triangle with vertices $(1,0,0),(0,1,0)$ and $(0,0,1)$ as indicted in Figure \ref{softmax} on the right. Hence the task of the network will be to deform this annulus to $\Delta_2$ in a way that respects the labels indicated by the colors.

\begin{figure}[h]
  \centering
   {\includegraphics[width=0.7\textwidth]{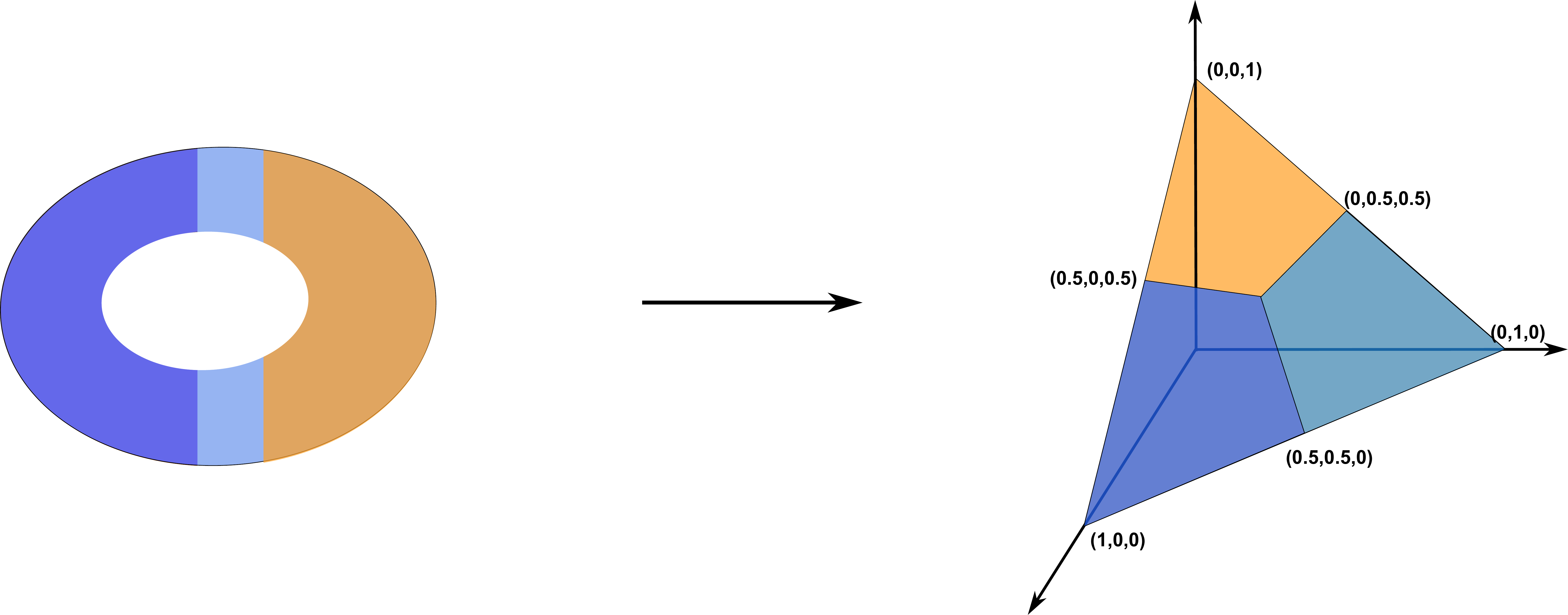}
    \caption{The softmax codomain for a dataset labeled with three labels.}
  \label{softmax}}
\end{figure}

In the example from Figure \ref{softmax} one can see that the neural network is trying to collapse all the topological information encoded by the original input space into a topological disk $\Delta_2$. Notice that the only geometric information that survives is related to the way the input topological space is labeled.  In order to address this point more precisely we need to recall the definition of Vonornoi sets and show the relationship of these sets to the softmax function.





\subsubsection{Softmax Activation Functions and Voronoi Sets}

We recall quickly the definition of a Voronoi diagram on general metric spaces. 

Let $(X,d)$ be a metric space and let $C\subset X$ be subset of $X$, called the \textit{the subset of centroids}. The \textit{Voronoi cell} at point $c\in C$, denoted by $VC(c)$ is defined to be the set of all points $y\in X$  that are closer to $c$ than to any other point in $C$. The collection of subsets $VC(c)$ for all $c$ in $C$ is by definition the Voronoi diagram, denoted by $VD(C)$ of the metric space $X$ with respect to the subset $C$.

These Voronoi diagrams are intimately related to the way we determine how a classification neural network determines the label of an input point.  Let $Net$ be a classification neural network with $n$ labels.  For an input $x\in \mathcal{X}$ the point $Net(x)$ is an element of $\Delta_{n-1}$. Recall that each vertex $v_i \in \Delta_{n-1}$ has a corresponding label $l_{i+1} \in \mathcal{Y}$. By definition, the point $x$ is assigned to the label $l_{i+1}$ by the neural network if and only if $Net(x) \in Int(VC( v_i ))$. Here $Int(A)$ denotes the interior of a set $A$. 

In other words, we divide the space $\Delta_{n-1}$ into $n$ disjoint sets, each one associated with a vertex of the simplex $\Delta_{n-1}$ (and thus with a unique label). These disjoint sets are precisely the Vonornoi sets of the vertices $\{v_i\}_{i=0}^{n-1}$.

So the task of a softmax classification neural network, viewed as a continuous map can be viewed as follows:

A softmax classification neural network tries to deform the input topological space $\mathcal{X}$ to the space $\Delta_{n-1}$ such that each subset $\mathcal{X}_{i+1}$ maps to the interior of the Voronoi cell $Int(VC(v_i))$ via a finite sequence of continuous topological operations: scaling, translation, rotation, reflection, bending and quotienting. Observe that this is consistent with Definition \ref{def1}. We record this in the following theorem.

\begin{theorem}
Let $(X,h,g)$ by a topologically labeled data with $h: X \longrightarrow R^{d_{in}}$ and $g: \mathcal{X} \subset \mathbb{R}^{d_{in}} \longrightarrow \{l_1, \cdots l_n \}$. A softmax classification neural network $Net : \mathbb{R}^{d_{in}} \to Int(\Delta_{n-1}) $ separates $(X,h,g)$ if and only if $Net(\mathcal{X}_{i+1}) \subset Int( VC(v_i)) $ for $0\leq i \leq n-1$.
\end{theorem}

\begin{remark}
It is worth mentioning here that while the softmax classification neural network's final goal is to map $Net(\mathcal{X}_{i+1})$ to $Int( VC(v_i) )$ for $0\leq i \leq n-1$, the network does not manipulate each subset $\mathcal{X}_{i+1}$ independently, but rather must manipulate the entire space $\mathcal{X}^{[k]}$ at each layer $k$ to achieve the mapping from $\mathcal{X}^{[0]}$ to $Int(\Delta_{n-1} )$ in a way that guarantees each labeled subset of the input space is mapped to the correct cell in the output simplex.
\end{remark}



\section{Topology and the parameter landscape of a deep network}
\label{topolgist}
Given the above discussion, one may think of a neural network as a ``topologist" that is trying to deform a space $A$ to another space $B$ by a finite sequence of five topological operations: bending, scaling, translating, rotating and quotienting. Indeed in a Relu neural network, each layer function acts on its input space by a combination of the above continuous operations.  

On the other hand, given two spaces $A$ and $B$ two topologists might choose to deform the space $A$ to the space $B$ via two different sequences of continuous operations. For instance, in Figure \ref{second_example} the first operation is a projection, which is essentially a quotienting operation, of the torus into a topological $2$-disk. This step can be done in infinitely many topologically equivalent manners. 

This observation is very closely related to the fact that the parameter landscape of a deep net has a large number of local minima that can consistently provide similar levels of performance on multiple experiments \cite{choromanska2015loss}. Indeed, local minima are likely to have an error very close to that of the global minimum \cite{sagun2015explorations,choromanska2015loss}. From this perspective and given the topological setting, it seems reasonable to conjecture that equivalent local minima correspond to distinct sequences of continuous operations that yield topologically equivalent sequences of spaces. 

\section{Concluding Remarks}
\label{conclusions}

In this section we discuss a few important topological considerations that were omitted in our analysis above. The first point is that the analogy made in Subsection \ref{topolgist} between a neural network and a topologist is not always accurate as demonstrated by the following example. Consider a new neural network $Net$ defined on the same dataset given in Figure \ref{second_example} and given by  $Net =Net_1 \circ f_1 $ where $f_1 : \mathbb{R}^3 \longrightarrow \mathbb{R}^2$. We saw in Example \ref{example123} that the matrix $W_1$ naturally projects the torus to a disk in this step, so it may seem reasonable to reduce the dimension of our output after the first layer. However, on several training experiments $Net$ given above yielded the projection given in Figure \ref{failing}.

\begin{figure}[h]
  \centering
   {\includegraphics[width=0.7\textwidth]{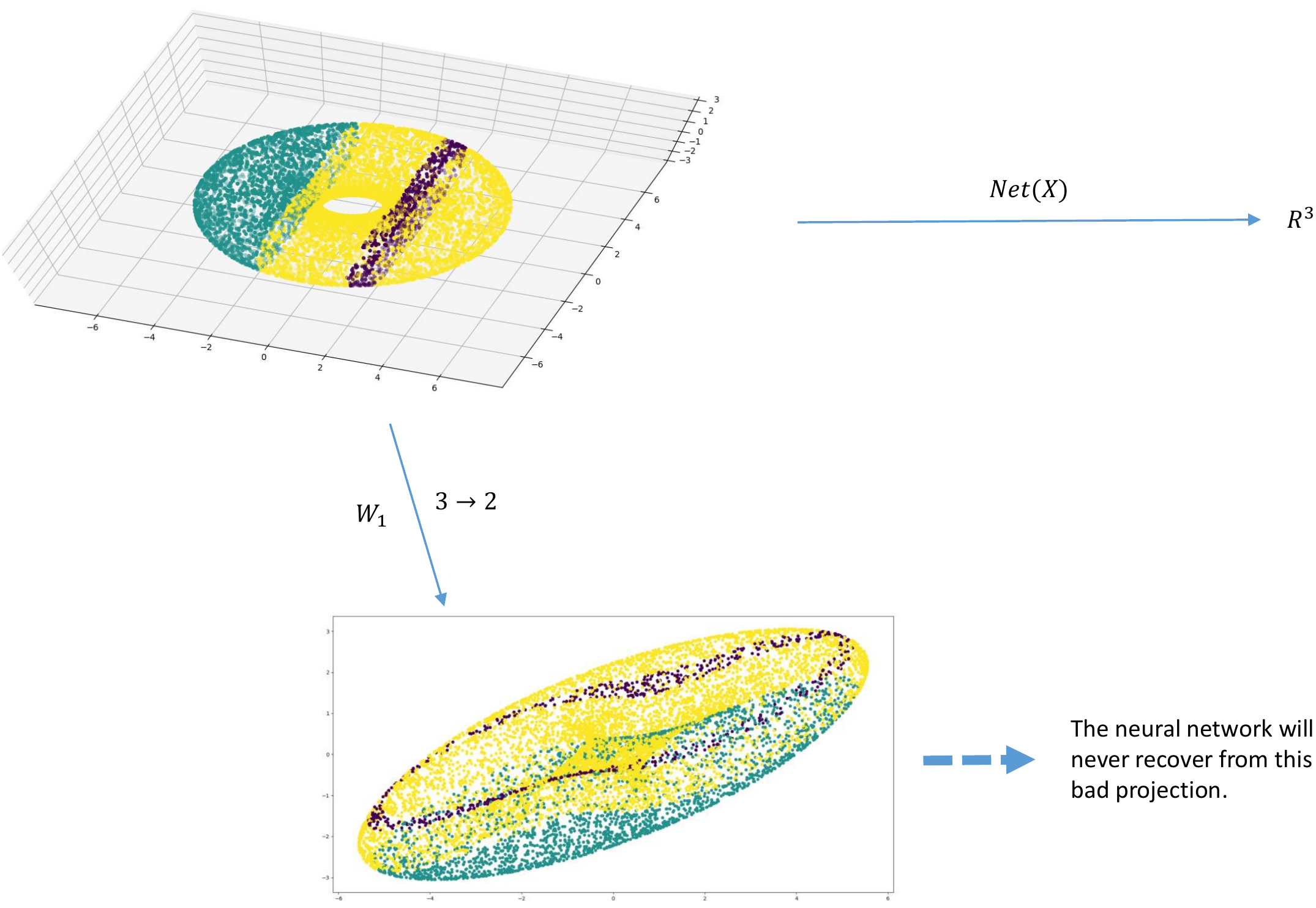}
    \caption{A network whose first map goes to $\mathbb{R}^2$ might lead the network, after training, to choose a projection that yields a high proportion of mislabled classifications. The network cannot recover from such "bad choice" of mapping, no matter how deep we try to make the network. }
  \label{failing}}
\end{figure}

The problem with such projections is that the neural network cannot recover from it no matter how deep we make it, because the parts of the topological space corresponding to the different labels have become mutually immersed. Moreover, this demonstrates how neural networks in general need more room to manipulate the space than a topologist might require. 



\subsection{The role of Linking and Knotting In Data}

We briefly discuss the role of knotting and linking when working with data and neural networks.

\subsubsection{Knotting}



For our purposes, the following results can be combined to provide an upper bound for the number of dimensions required for a softmax classification neural networks to separate topological labelled data.  

The General Position Theorem asserts that any $m$-manifold unknots in $R^n$ provided $n\geq 2m+2$. The theorem holds in the smooth and piece-wise linear setting \cite{rourke2012introduction}. This means that if two PL-embeddings $f,g:M^k\hookrightarrow N^n$ of a compact PL manifold of dimension $k$ are homotopic and $n\geq 2k+2$, then they are PL-isotopic.  We can adapt this result to fit our framework in the following proposition.

\begin{proposition}
\label{unknottingMultiples}
Let $X = M_1^{i_1} \cupdot M_2^{i_2} \cdots  \cupdot M_k^{i_k}$ be a disjoint union of $k$ compact manifolds. Let $h_1,h_2: X \longrightarrow \mathbb{R}^n$ be two PL-linear embeddings of $X$. If $h_1$ and $h_2$ are homotopic and $n\geq 2\times \max\{i_1,\cdots i_k\}+2 $ then $h_1$ and $h_2$ are isotopic.
\end{proposition}

The general position theorem asserts all embeddings of the manifold $M^k$ are essentially equivalent as long as $k$ is sufficiently large. In our setting, this means that data is easier to manipulate once it is embedded in higher dimensions. This explains our choice of some of the layer functions in Example \ref{examples1}.



\subsubsection{Linking}
Unlinking between two general manifolds \cite{332289} can be given as follows.

\begin{definition}
\label{linking1}
Two disjoint submanifolds $A$ and $B$ of a manifold $M^n$ are unlinked if we can find disjoint embedded $m$-dimensional discs $D_1,D_2 \subset M$ such that $A\subset D_1$ and $B \subset D_2$.  The manifolds $A$ and $B$ are said to be \textit{linked} if they are not unlinked.
\end{definition}


Note that this definition is related to our definition of a separating function in Definition \ref{def1}. However, the conditions for being unlinked are more strict than the conditions we set in Definition \ref{def1}.


In our experimentation we observed that manipulating data is easier when we give the layers of the neural network a dimension close to the dimension of the general position theorem, or higher, especially in the first few layers where the network requires "extra room" to perform manipulations to untangle the space. We plan to investigate this observation further in a future study. We conclude our discussion with an example about this point.

\begin{example}

Consider the network $Net: \mathbb{R}^3\longrightarrow \mathbb{R}^2 $ defined by the composition $Net=f_5\circ f_4 \circ f_3 \circ f_2 \circ f_1$ where $f_1:\mathbb{R}^3\longrightarrow \mathbb{R}^7 $, $f_2:\mathbb{R}^7\longrightarrow \mathbb{R}^7 $, $f_3:\mathbb{R}^7\longrightarrow \mathbb{R}^7 $, $f_4:\mathbb{R}^7\longrightarrow \mathbb{R}^3 $ and $f_5:\mathbb{R}^3\longrightarrow \mathbb{R}^2 $. The activations are described in Figure \ref{unlinking}.

\begin{figure}[h]
  \centering
   {\includegraphics[width=\textwidth]{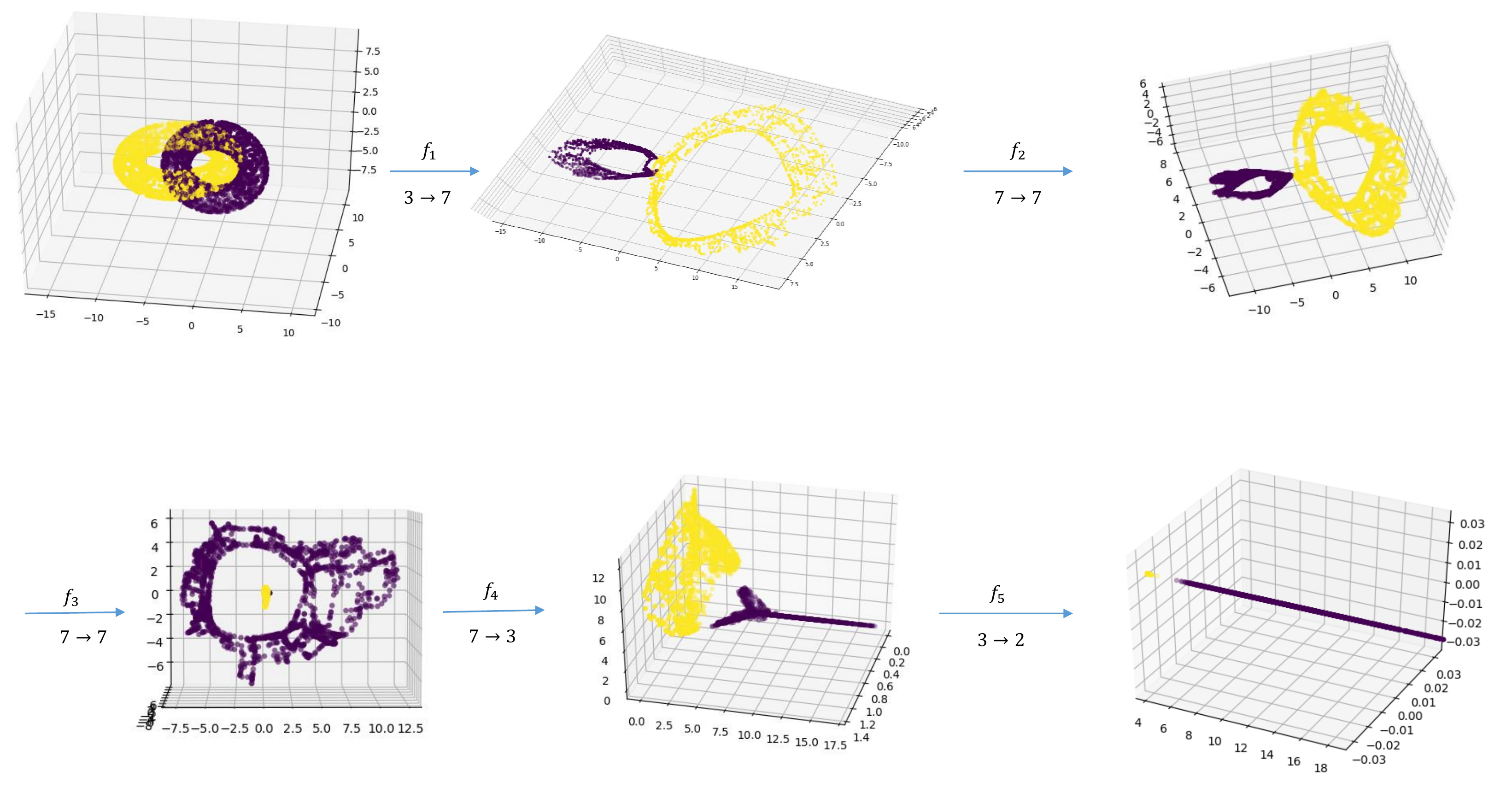}
    \caption{Unlinking the data using a neural net.}
  \label{unlinking}}
\end{figure}

Note that the neural network cannot unlink the two tori given in the input if we do not go to a higher dimension. Moreover, to visualize the higher dimensional activations in Figure \ref{unlinking} we project them to $\mathbb{R}^3$ using Isomap \cite{tenenbaum2000global} as we described in Example \ref{examples1}.

\end{example}

\bibliographystyle{abbrv}

\bibliography{refs_2}

\end{document}